%% file: main.tex
\documentclass{article}
\usepackage{amsfonts}
\usepackage{amsmath}
\usepackage{commath}
\usepackage{mathbbol}
\usepackage{cancel}
\usepackage{multirow, pbox}

\usepackage{times}
\usepackage{graphicx} 
\usepackage{subfigure} 

\usepackage{natbib}

\usepackage{algorithm}
\usepackage{algorithmic}

\usepackage{amsthm,amssymb}

\newtheorem{proposition}{Proposition}

\newtheorem{theorem}{Theorem}
\newtheorem{corollary}{Corollary}

\usepackage{hyperref}

\usepackage[backgroundcolor=White,textwidth=0.8in]{todonotes}
\input{def}

\usepackage[accepted]{icml2018}

\icmltitlerunning{Path Consistency Learning in Tsallis Entropy Regularized MDPs}

\begin{document} 

\twocolumn[
\icmltitle{Path Consistency Learning in Tsallis Entropy Regularized MDPs}



\icmlsetsymbol{equal}{*}

\begin{icmlauthorlist}
\icmlauthor{Ofir Nachum}{equal,gbrain}
\icmlauthor{Yinlam Chow}{equal,dm}
\icmlauthor{Mohamamd Ghavamzadeh}{equal,dm}
\end{icmlauthorlist}

\icmlaffiliation{gbrain}{Google Brain}
\icmlaffiliation{dm}{Google DeepMind}

\icmlcorrespondingauthor{Yinlam Chow}{yinlamchow@google.com}

\icmlkeywords{Reinforcement Learning, Path Consistency Learning, Tsallis Entropy}

\vskip 0.3in
]



\printAffiliationsAndNotice{\icmlEqualContribution} 

\begin{abstract} 
We study the sparse entropy-regularized reinforcement learning (ERL) problem in which the entropy term is a special form of the {\em Tsallis} entropy. The optimal policy of this formulation is sparse, i.e.,~at each state, it has non-zero probability for only a small number of actions. This addresses the main drawback of the standard Shannon entropy-regularized RL (soft ERL) formulation, in which the optimal policy is {\em softmax}, and thus, may assign a non-negligible probability mass to non-optimal actions. This problem is aggravated as the number of actions is increased. 
In this paper, we follow the work of~\citet{pcl} in the soft ERL setting, and propose a class of novel path consistency learning (PCL) algorithms, called {\em sparse PCL}, for the sparse ERL problem that can work with both on-policy and off-policy data. We first derive a {\em sparse consistency} equation that specifies a relationship between the optimal value function and policy of the sparse ERL along any system trajectory. Crucially, a weak form of the converse is also true, and we quantify the sub-optimality of a policy which satisfies sparse consistency, and show that as we increase the number of actions, this sub-optimality is better than that of the soft ERL optimal policy. We then use this result to derive the sparse PCL algorithms. We empirically compare sparse PCL with its soft counterpart, and show its advantage, especially in problems with a large number of actions.
\end{abstract} 

\vspace{-0.3in}
\section{Introduction} \label{sec:intro}
\vspace{-0.025in}
\input{intro}

\vspace{-0.1in}
\section{Markov Decision Processes (MDPs)}\label{sec:mdp}
\vspace{-0.05in}
\input{mdp}

\vspace{-0.1in}
\section{Entropy Regularized MDPs}\label{sec:ermdp}
\vspace{-0.05in}
\input{ermdp}

\section{Path Consistency Learning in Soft MDPs}\label{sec:pcl-soft}
\input{pcl-soft}

\section{Consistency between Optimal Value \& Policy in Sparse MDPs}\label{sec:consistency-sparse}
\input{consistency-sparse}

\section{Path Consistency Learning in Sparse MDPs}\label{sec:pcl-sparse}
\input{pcl-sparse}

\section{Experimental Results}\label{sec:experiment}
\input{experiment}

\section{Conclusions}\label{sec:conclusions} 
\input{conclusion}

\newpage
\bibliography{ref}
\bibliographystyle{icml2018}

\newpage
\appendix
\onecolumn

\section{Proofs of Section~\ref{sec:consistency-sparse}}\label{sec:appendix-proofs}
\input{proofs}

\section{Experimental Details}
\input{details}

 
\end{document}

%% file: def.tex
\newcommand{\X}{\mathcal{X}}

\newcommand{\A}{\mathcal{A}}
\newcommand{\M}{\mathcal{M}}

%% file: intro.tex

In reinforcement learning (RL), the goal is to find a policy with maximum long-term performance, defined as the sum of discounted rewards generated by following the policy~\citep{Bertsekas96NP,Sutton98IR}. In case the number of states and actions are small, and the model is known, the optimal policy is the solution of the {\em non-linear} Bellman optimality equations~\citep{Bellman57DP}. When the system is large or the model is unknown, greedily solving the Bellman equations often results in policies that are far from optimal. A principled way of dealing with this issue is {\em regularization}. Among different forms of regularization, such as $\ell_2$ (e.g.,~\citealt{Farahmand08RP,Farahmand09RF}) and $\ell_1$ (e.g.,~\citealt{Kolter09RF,Johns10LC,Ghavamzadeh11FS}), {\em entropy regularization} is among the most studied in both value-based (e.g.,~\citealt{Kappen05PI,Todorov06LS,Ziebart10MP,Azar12DP,Fox16GL,ODonoghue17CP,Asadi17AS}) 
and policy-based (e.g.,~\citealt{Peters10RE,Todorov10PG}) RL formulations. In particular, two of the most popular deep RL algorithms, TRPO~\citep{Schulman15TR} and A3C~\citep{Mnih16AM}, are based on entropy-regularized policy search. We refer the interested readers to~\citet{Neu17UV}, for an insightful discussion on entropy-regularized RL algorithms and their connection to online learning.  

In entropy-regularized RL (ERL), an entropy term is added to the Bellman equation. This formulation has four main advantages: {\bf 1)} it softens the non-linearity of the Bellman equations and makes it possible to solve them more easily, {\bf 2)} the solution of the softened problem is quantifiably not much worse than the optimal solution in terms of accumulated return, {\bf 3)} the addition of the entropy term brings nice properties, such as encouraging exploration (Shannon entropy) (e.g.,~\citealt{Fox16GL,pcl}) and maintaining a close distance to a baseline policy (relative entropy) (e.g.,~\citealt{Schulman15TR, tpcl}), and {\bf 4)} unlike the original problem that has a deterministic solution, the solution to the softened problem is stochastic, which is preferable in problems in which exploration or dealing with unexpected situations is important. However, in the most common form of ERL, in which a Shannon (or relative) entropy term is added to the Bellman equations, the optimal policy is of the form of {\em softmax}. Despite the advantages of a softmax policy in terms of exploration, its main drawback is that at each step, it assigns a non-negligible probability mass to non-optimal actions, a problem that is aggravated as the number of actions is increased. This may result in policies that may not be safe to execute. To address this issue,~\citet{Lee18SM} proposed to add a special form of a general notion of entropy, called Tsallis entropy~\citep{Tsallis88PG}, to the Bellman equations. This formulation has the property that its solution has sparse distributions, i.e.,~at each state, only a small number of actions have non-zero probability.~\citet{Lee18SM} studied the properties of this ERL formulation, proposed value-based algorithms (fitted Q-iteration and Q-learning) to solve it, and showed that although it is harder to solve than its soft counterpart, it potentially has a solution closer to that of the original problem. 

In this paper, we propose novel path consistency learning (PCL) algorithms for the Tsallis ERL problem, called {\em sparse PCL}. PCL is a class of actor-critic type algorithms developed by~\citet{pcl} for the soft (Shannon entropy) ERL problem. It uses a nice property of soft ERL, namely the equivalence of consistency and optimality, and learns parameterized policy and value functions by minimizing a loss that is based on the consistency equation of soft ERL. The most notable feature of soft PCL is that it can work with both on-policy (sub-trajectories generated by the current policy) and off-policy (sub-trajectories generated by a policy different than the current one, including any sub-trajectory from the replay buffer) data. We first derive a multi-step consistency equation for the Tsallis ERL problem, called {\em sparse consistency}. We then prove that in this setting, while optimality implies consistency (similar to the soft case), unlike the soft case, consistency only implies sub-optimality. We then use the sparse consistency equation and derive PCL algorithms that use both on-policy and off-policy data to solve the Tsallis ERL problem. We empirically compare sparse PCL with its soft counterpart. As expected, we gain from using the sparse formulation when the number of actions is large, both in algorithmic tasks and in discretized continuous control problems.

%% file: mdp.tex

We consider the reinforcement learning (RL) problem in which the agent's interaction with the system is modeled as a MDP. A MDP is a tuple $\M=(\X,\A,r,P,P_0,\gamma)$, where $\X$ and $\A$ are state and action spaces; $r:\X\times\A\rightarrow\mathbb{R}$ and $P:\X\times\A\rightarrow\Delta_\X$ are the reward function and transition probability distribution, with $r(x,a)\in[0,R_{\max}]$ and $P(\cdot|x,a)$ being the reward and the next state probability of taking action $a$ in state $x$; $P_0:\X\rightarrow\Delta_\X$ is the initial state distribution; and $\gamma\in[0,1)$ is a discounting factor. In this paper, we assume that the action space is finite, but can be large. The goal in RL is to find a stationary Markovian policy, i.e.,~a mapping from state and action spaces to a simplex over the actions $\mu:\X\times\A\rightarrow\Delta_\A$, that maximizes the expected discounted sum of rewards, i.e.,

\vspace{-0.25in}
\begin{small}
\begin{align}
\label{eq:mdp-opt}
&\max_\mu \; \mathbb{E}\big[\sum_{t=0}^\infty\gamma^tr(x_t,a_t)\big] \\
&\text{s.t.} \quad\; \forall x\;\; \sum_{a\in\A}\mu(a|x)=1,\quad\forall x,a\;\; \mu(a|x)\geq 0, \nonumber 
\end{align}
\end{small}
\vspace{-0.25in}
  
where $x_0\sim P_0$, $a_t\sim\mu(\cdot|x_t)$, and $x_{t+1}\sim P(\cdot|x_t,a_t)$. 
For a given policy $\mu$, we define its value and action-value functions as 

\vspace{-0.275in}
\begin{small}
\begin{align*}
V^\mu(x) &= \mathbb{E}\big[\sum_{t=0}^\infty\gamma^tr(x_t,a_t)|x_0=x,\mu,P\big], \\
Q^\mu(x,a) &= \mathbb{E}\big[\sum_{t=0}^\infty\gamma^tr(x_t,a_t)|x_0=x,a_0=a,\mu,P\big]. 
\end{align*}
\end{small}
\vspace{-0.175in}

%
%

Any solution of the optimization problem~\eqref{eq:mdp-opt} is called an {\em optimal} policy and is denoted by $\mu^*$. Note that while a MDP may have several optimal policies, it only has a single optimal value function $V^*=V^{\mu^*}$. It has been proven that~\eqref{eq:mdp-opt} has a solution in the space of {\em deterministic} policies, i.e.,~$\Pi_d=\{\mu:\mu:\X\rightarrow\A\}$, which can be obtained as the {\em greedy} action w.r.t.~the optimal action-value function, i.e.,~$\mu^*(x)\in\arg\max_{a}Q^*(x,a)$~\citep{Puterman94MD,Bertsekas96NP}. The optimal action-value function $Q^*$ is the {\em unique} solution of the non-linear Bellman optimality equations, i.e.,~for all $x\in\X$ and $a\in\A$,
\begin{equation}
\label{eq:Bellman-optimality}
Q(x,a) = r(x,a) + \gamma\sum_{x'\in\X}P(x'|x,a)\max_{a'\in\A}Q(x',a').
\end{equation}
Any optimal policy $\mu^*$ and the optimal state and state-action value functions, $V^*$ and $Q^*$, satisfy the following equations for all states and action,

\vspace{-0.15in}
\begin{small}
\begin{align*}
&Q^*(x,a) = r(x,a) + \gamma\sum_{x'\in\X}P(x'|x,a)V^*(x'), \\
&V^*(x) = \max_{a\in\A} Q^*(x,a), \quad\; \mu^*(x) \in \arg\max_{a\in\A} Q^*(x,a).
\end{align*}
\end{small}
\vspace{-0.1in}

%% file: ermdp.tex

As discussed in Section~\ref{sec:mdp}, finding an optimal policy for a MDP involves solving a non-linear system of equations (see Eq.~\ref{eq:Bellman-optimality}), which is often complicated. Moreover, the optimal policy may be deterministic, always selecting the same optimal action at a state even when there are several optimal actions in that state. This is undesirable when it is important to explore and to deal with unexpected situations. In such cases, one might be interested in multimodal policies that still have good performance. This is why many researchers have proposed to add a regularizer in the form of an {\em entropy} term to the objective function~\eqref{eq:mdp-opt} and solve the following {\em entropy-regularized} optimization problem  

\vspace{-0.2in}
\begin{small}
\begin{align}
\label{eq:mdp-reg-opt}
&\max_\mu \; \mathbb{E}\Big[\sum_{t=0}^\infty\gamma^t\big(r(x_t,a_t) + \alpha H^\mu(x_t,a_t)\big)\Big] \\
&\text{s.t.} \quad\; \forall x\;\; \sum_{a\in\A}\mu(a|x)=1,\quad\forall x,a\;\; \mu(a|x)\geq 0, \nonumber 
\end{align}
\end{small}
\vspace{-0.15in}

where $H^\mu(x,a)$ is an entropy-related term and $\alpha$ is the regularization parameter. The entropy term smoothens the objective function~\eqref{eq:mdp-opt} such that the resulting problem~\eqref{eq:mdp-reg-opt} is often easier to solve than the original one~\eqref{eq:mdp-opt}. This is another reason for the popularity of entropy-regularized MDPs.


\subsection{Entropy Regularized MDP with Shannon Entropy}
\label{subsec:emdp-soft}

It is common to use $H^\mu_{\text{sf}}(x_t,a_t)\stackrel{\triangle}{=}-\log\mu(a_t|x_t)$ in entropy-regularized MDPs (e.g.,~\citealt{Fox16GL,pcl}). Note that $H_{\text{sf}}(\mu)=\mathbb{E}_\mu\big[H^\mu_{\text{sf}}(x,a)\big]$ is the {\em Shannon entropy}. Problem~\eqref{eq:mdp-reg-opt} with $H^\mu_{\text{sf}}(x,a)$ can be seen as a RL problem in which the reward function is the sum of the original reward function $r(x,a)$ and a term that encourages {\em exploration}.\footnote{Another entropy term that has been studied in the literature is $H^\mu_{\text{rel}}(x_t,a_t)\stackrel{\triangle}{=}-\log\frac{\mu(a_t|x_t)}{\mu_b(a_t|x_t)}$, where $\pi_b$ is a baseline policy. Note that $H_{\text{rel}}(\mu)=\mathbb{E}_\mu\big[H^\mu_{\text{rel}}(x,a)\big]$ is the {\em relative entropy}. Problem~\eqref{eq:mdp-reg-opt} with $H^\mu_{\text{rel}}(x,a)$ can be seen as a RL problem in which the reward function is the sum of the original reward function $r(x,a)$ and a term that penalizes deviation from the baseline policy $\pi_b$.} Unlike~\eqref{eq:mdp-opt}, the optimization problem~\eqref{eq:mdp-reg-opt} with $H_{\text{sf}}^\mu$ has a unique optimal policy $\mu^*_{\text{sf}}$ and a unique optimal value $V^*_{\text{sf}}$ (action-value $Q^*_{\text{sf}}$) function that satisfy the following equations:

\vspace{-0.2in}
\begin{small}
\begin{align}
\label{eq:soft-opt}
Q_{\text{sf}}^*(x,a) &= r(x,a) + \gamma\sum_{x'\in\X}P(x'|x,a)V_{\text{sf}}^*(x'), \nonumber \\
V_{\text{sf}}^*(x) &= \alpha \cdot \text{sfmax}\big(Q_\text{sf}^*(x,\cdot)/\alpha\big), \\
\mu_{\text{sf}}^*(a|x) &= \frac{\exp\big(Q_{\text{sf}}^*(x,a)/\alpha\big)}{\sum_{a'\in\A}\exp\big(Q_{\text{sf}}^*(x,a')/\alpha\big)}, \nonumber
\end{align}
\end{small}
\vspace{-0.15in}

where for any function $f:\X\times\A\rightarrow\mathbb{R}$, the sfmax operator is defined as $\text{sfmax}\big(f(x,\cdot)\big) = \log\big(\sum_{a}\exp\big(f(x,a)\big)\big)$. Note that the equations in~\eqref{eq:soft-opt} are derived from the KKT conditions of~\eqref{eq:mdp-reg-opt} with $H^\mu_{\text{sf}}$. In this case, the optimal policy is {\em soft-max}, with the regularization parameter $\alpha$ playing the role of its temperature (see Eq.~\ref{eq:soft-opt}). This is why~\eqref{eq:mdp-reg-opt} with $H^\mu_{\text{sf}}$ is called the {\em soft MDP} problem. In soft MDPs, the optimal value function $V^*_{\text{sf}}$ is the unique solution of the {\em soft Bellman optimality} equations, i.e.,~$\forall x\in\X,\forall a\in\A$,

\vspace{-0.15in}
\begin{small}
\begin{equation}
\label{eq:soft-Bellman-optimality}
V(x) = \alpha \cdot \text{sfmax}\Big(\big[r(x,\cdot)+\gamma\sum_{x'}P(x'|x,\cdot)V(x')\big]/\alpha\Big).
\end{equation}
\end{small}
\vspace{-0.15in}

Note that the $\text{sfmax}$ operator is a smoother function of its inputs than the $\max$ operator associated with the Bellman optimality equation~\eqref{eq:Bellman-optimality}.
This means that solving the soft MDP problem is easier than the original one, with the cost that its optimal policy $\mu^*_{\text{sf}}$ performs worse than the optimal policy of the original MDP $\mu^*$. This difference can be quantified as 

\vspace{-0.15in}
\begin{small}
\begin{equation}
\label{eq:perf-soft}
\forall x\in\X \quad V^*(x) - \frac{\alpha}{1-\gamma}\log(|\A|)\leq V^{\mu^*_{\text{sf}}}(x) \leq V^*(x),
\end{equation}
\end{small}
\vspace{-0.15in}

where we discriminate between the value function of a policy $\mu$ in the soft $V^\mu_{\text{sf}}$ and original $V^\mu$ MDPs. Note that the sub-optimality of $\mu^*_{\text{sf}}$ is unbounded as $|\A|\to\infty$. This is the main drawback of using softmax policies; in large action space problems, at each step, the policy assigns a non-negligible probability mass to non-optimal actions, which in aggregate can be detrimental to its reward performance.

\subsection{Entropy Regularized MDP with Tsallis Entropy}
\label{subsec:emdp-sparse}

To address the issues with the softmax policy,~\citet{Lee18SM} proposed to use $H^\mu_{\text{sp}}(x_t,a_t)\stackrel{\triangle}{=}\frac{1}{2}\big(1-\mu(a_t|x_t)\big)$ in entropy-regularized MDPs. Note that $H_{\text{sp}}(\mu)=\mathbb{E}_\mu\big[H^\mu_{\text{sp}}(x,a)\big]$ is a special case of a general notion of entropy, called {\em Tsallis entropy}~\citep{Tsallis88PG}, i.e.,~$S_{q,k}(p)=\frac{k}{q-1}(1-\sum_ip_i^q)$, for the parameters $q=2$ and $k=\frac{1}{2}$.\footnote{Note that the Shannon entropy is a special case of the Tsallis entropy for the parameters $q=k=1$~\citep{Tsallis88PG}.} Similar to the soft MDP problem, the optimization problem~\eqref{eq:mdp-reg-opt} with $H_{\text{sp}}^\mu$ has a unique optimal policy $\mu^*_{\text{sp}}$ and a unique optimal value $V^*_{\text{sp}}$ (action-value $Q^*_{\text{sp}}$) function that satisfy the following equations~\citep{Lee18SM}:  

\vspace{-0.15in}
\begin{small}
\begin{align}
\label{eq:sparse-opt}
Q_{\text{sp}}^*(x,a) &= r(x,a) + \gamma\sum_{x'\in\X}P(x'|x,a)V_{\text{sp}}^*(x'), \nonumber \\
V_{\text{sp}}^*(x) &= \alpha \cdot \text{spmax}\big(Q_\text{sp}^*(x,\cdot)/\alpha\big), \nonumber \\
\mu_{\text{sp}}^*(a|x) &= \Big(Q^*_{\text{sf}}(x,a)/\alpha - \mathcal{G}\big(Q^*_{\text{sf}}(x,\cdot)/\alpha\big)\Big)^+, 
\end{align}
\end{small}
\vspace{-0.15in}

where $(\cdot)^+=\max(\cdot,0)$, and for any function $f:\X\times\A\rightarrow\mathbb{R}$, the spmax operator is defined as

\vspace{-0.15in}
\begin{small}
\begin{equation*}
\text{spmax}\big(f(x,\cdot)\big) = \frac{1}{2}\Big[1 + \sum_{a\in\mathcal{S}(x)}\Big(\Big(\frac{f(x,a)}{\alpha}\Big)^2 - \mathcal{G}\Big(\frac{f(x,\cdot)}{\alpha}\Big)^2\Big)\Big],
\end{equation*}
\end{small}
\vspace{-0.15in}

in which

\vspace{-0.15in}
\begin{small}
\begin{equation*}
\mathcal{G}\big(f(x,\cdot)\big) = \frac{\sum_{a\in\mathcal{S}(x)}f(x,a) - 1}{|\mathcal{S}(x)|}
\end{equation*}
\end{small}
\vspace{-0.2in}

and $\mathcal{S}(x)$ is the set of actions satisfying $1+i\frac{f(x,a_{(i)})}{\alpha}>\sum_{j=0}^i\frac{f(x,a_{(j)})}{\alpha}$, where $a_{(i)}$ indicates the action with the $i$th largest value of $f(x,a)$. Note that the equations in~\eqref{eq:sparse-opt} are derived from the KKT conditions of~\eqref{eq:mdp-reg-opt} with $H^\mu_{\text{sp}}$. In this case, the optimal policy may have zero probability for several actions (see Eq.~\ref{eq:sparse-opt}). This is why~\eqref{eq:mdp-reg-opt} with $H^\mu_{\text{sp}}$ is called the {\em sparse MDP} problem. The regularization parameter $\alpha$ controls the sparsity of the resulted policy. The policy would be more sparse for smaller values of $\alpha$. In sparse MDPs, the optimal value function $V^*_{\text{sp}}$ is the unique fixed-point of the {\em sparse Bellman optimality} operator $\mathcal{T}_{\text{sp}}$~\citep{Lee18SM} that for any function $f:\X\rightarrow\mathbb{R}$ is defined as

\vspace{-0.2in}
\begin{small}
\begin{equation}
\label{eq:sparse-Bellman-optimality}
(\mathcal{T}_{\text{sp}}f)(x) = \alpha \cdot \text{spmax}\Big(\big[r(x,\cdot)+\gamma\sum_{x'}P(x'|x,\cdot)f(x')\big]/\alpha\Big).
\end{equation}
\end{small}
\vspace{-0.2in}

Similar to~\eqref{eq:soft-Bellman-optimality}, the $\text{spmax}$ operator is a smoother function of its inputs than the $\max$, and thus, solving the sparse MDP problem is easier than the original one, with the cost that its optimal policy $\mu^*_{\text{sp}}$ performs worse than the optimal policy of the original MDP $\mu^*$. This difference can be quantified as~\citep{Lee18SM},

\vspace{-0.2in}
\begin{small}
\begin{equation}
\label{eq:perf-sparse}
\forall x\in\X \quad V^*(x) - \frac{\alpha}{1-\gamma}\cdot\frac{|\mathcal{A}|-1}{2|\mathcal{A}|}\leq V^{\mu^*_{\text{sp}}}(x) \leq V^*(x).
\end{equation}
\end{small}
\vspace{-0.25in}


On the other hand, the spmax operator is more complex than sfmax, and thus, it is slightly more complicated to solve the sparse MDP problem than its soft counterpart. However, as can be seen from Eqs.~\ref{eq:perf-soft} and~\ref{eq:perf-sparse}, the optimal policy of the sparse MDP, $\mu^*_{\text{sp}}$, can have a better performance than its soft counterpart, $\mu^*_{\text{sf}}$, and this difference becomes more apparent as the number of actions $|\mathcal{A}|$ grows. For large action size, the term $(|\mathcal{A}|-1)/(2|\mathcal{A}|)$ in~\eqref{eq:perf-sparse} turns to a constant, while $\log|\mathcal{A}|$ in~\eqref{eq:perf-soft} grows unbounded.

%% file: pcl-soft.tex

A nice property of soft MDPs that was elegantly used by~\citet{pcl} is that any policy $\mu$ and function $V:\X\rightarrow\mathbb{R}$ that satisfy the (one-step) {\em consistency} equation, i.e.,~for all $x\in\X$ and for all $a\in\A$,

\vspace{-0.2in}
\begin{small}
\begin{equation}
\label{eq:single-step-consistency-soft}
V(x) = r(x,a) - \alpha\log\mu(a|x) + \gamma\sum_{x'\in\X}P(x'|x,a)V(x'),
\end{equation}
\end{small}
\vspace{-0.2in}

are optimal, i.e.,~$\mu=\mu^*_{\text{sf}}$ and $V=V^*_{\text{sf}}$ ({\em consistency implies optimality}). Due to the uniqueness of the optimal policy in soft MDPs, the reverse is also true, i.e.,~the optimal policy $\mu^*_{\text{sp}}$ and the value function $V^*_{\text{sp}}$ satisfy the {\em consistency} equation ({\em optimality implies consistency}). 

As shown in~\citet{pcl}, the (one-step) consistency equation~\eqref{eq:single-step-consistency-soft} can be easily extended to multi-step, i.e.,~any policy $\mu$ and function $V:\X\rightarrow\mathbb{R}$ that for any state $x_0$ and sequence of actions $a_0,\ldots,a_{d-1}$, satisfy the {\em multi-step consistency} equation

\vspace{-0.2in}
\begin{small}
\begin{align}
\label{eq:multi-step-consistency-soft}
V(x_0) &= \mathbb{E}_{x_{1:d}|x_0,a_{0:d-1}}\Big[\gamma^dV(x_d) \\
&+ \sum_{t=0}^{d-1}\gamma^t\big(r(x_t,a_t) - \alpha\log\mu(a_t|x_t)\big)\Big] \nonumber
\end{align}
\end{small}
\vspace{-0.2in}

are optimal, i.e.,~$\mu=\mu^*_{\text{sf}}$ and $V=V^*_{\text{sf}}$. 

The property that both single and multiple step consistency equations imply optimality (Eqs.~\ref{eq:single-step-consistency-soft} and~\ref{eq:multi-step-consistency-soft}) was the motivation of a RL algorithm by~\citet{pcl}, {\em path consistency learning} (PCL). The main idea of (soft) PCL is to learn a parameterized policy and value function by minimizing the following objective function:

\vspace{-0.1in}
\begin{small}
\begin{equation*}
\mathcal{J}(\theta,\phi) = \frac{1}{2}\sum_{\xi_i}J(\xi_i,\theta,\phi)^2,
\end{equation*}
\end{small}
\vspace{-0.175in}

where $\xi=(x_0,a_0,r_0,\ldots,x_{d-1},a_{d-1},r_{d-1},x_d)$ is any $d$-length sub-trajectory, $\theta$ and $\phi$ are the policy and value function parameters, respectively, and 

\vspace{-0.2in}
\begin{small}
\begin{align}
\label{eq:PCL-soft1}
J(\xi,\theta,\phi) &= -V_\phi(x_0) + \gamma^dV_\phi(x_d) \\
&+ \sum_{t=0}^{d-1}\gamma^t\big(r(x_t,a_t) - \alpha\log\mu_\theta(a_t|x_t)\big). \nonumber
\end{align}
\end{small}
\vspace{-0.15in}

An important property of the soft PCL algorithm is that since the multi-step consistency~\eqref{eq:multi-step-consistency-soft} holds for any $d$-length sub-trajectory, it can use both on-policy ($\xi$'s generated by the current policy $\mu_\theta$) and off-policy data, i.e.,~$\xi$'s generated by a policy different than the current one, including any $d$-length sub-trajectory from the replay buffer. 

Note that since both optimal policy $\mu^*_{\text{sf}}$ and value function $V^*_{\text{sf}}$ can be written based on the optimal action-value function $Q^*_{\text{sf}}$ (see Eq.~\ref{eq:soft-opt}), we may write the objective function~\eqref{eq:PCL-soft1} based on $Q_\psi$, and optimize only one set of parameters $\psi$, instead of separate $\theta$ and $\phi$.

%% file: consistency-sparse.tex

This section begins the main contributions of our work.
We first identify a (one-step) consistency equation for the sparse MDPs defined by~\eqref{eq:mdp-reg-opt}. We then prove the relationship between the {\em sparse consistency} equation and the optimal policy and value function of the sparse MDP, and highlight its similarities and differences with that in soft MDPs, discussed in Section~\ref{sec:pcl-soft}. 
We then extend the sparse consistency equation to multiple steps and 
prove results that allow us to use the {\em multi-step sparse consistency} 
equation to derive on-policy and off-policy algorithms to solve sparse MDPs, which we fully describe in Section~\ref{sec:pcl-sparse}. 
The significance of the sparse consistency equation is in providing an efficient tool 
for computing a {\em near-optimal} policy for sparse MDPs, 
which only involves solving a set of linear equations and linear complementary constraints, 
as opposed to (iteratively) solving the fixed-point of the non-linear sparse Bellman operator~\eqref{eq:sparse-Bellman-optimality}. 
We report the proofs of all the theorems of this section in Appendix~\ref{sec:appendix-proofs}.

For any policy $\mu$ and value function $V:\X\rightarrow\mathbb{R}$, we define the (one-step) consistency equation of the sparse MDPs as, for all state $x\in\X$ and for all actions $a\in\A$,    
%
\begin{align}
V(x) &= r(x,a) + \frac{\alpha}{2} - \alpha\mu(a|x) + \lambda(a|x) - \Lambda(x) \nonumber \\ 
&+ \gamma\sum_{x'}P(x'|x,a)V(x'),
\label{eq:consistency_1}
\end{align}
%
where $\lambda:\X\times\A\rightarrow\mathbb{R}_+$ and $\Lambda:\X\rightarrow\mathbb{R}_-$ are Lagrange multipliers, such that $\lambda(a|x)\cdot\mu(a|x)=0$ and $-\frac{\alpha}{2}\leq\Lambda(x)\leq 0$. We call this the {\em one-step sparse consistency} equation and it is the equivalent of Eq.~\ref{eq:single-step-consistency-soft} in soft MDPs. 

We now present a theorem which states that, similar to soft MDPs, 
optimality in sparse MDPs is a necessary condition for consistency, i.e.,~{\em optimality implies consistency}.

%


\begin{theorem}
\label{thm:consistency}
The optimal policy $\mu_{\text{sp}}^*$ and value function $V_{\text{sp}}^*$ of the sparse MDP~\eqref{eq:mdp-reg-opt} satisfy the consistency equation~\eqref{eq:consistency_1}.
\end{theorem}


Theorem~\ref{thm:perf_consistency} shows that in the sparse MDPs, {\em consistency only implies near-optimality}, as opposed to optimality in the case of soft MDPs.

%

\begin{theorem}
\label{thm:perf_consistency}
Any policy $\mu$ that satisfies the consistency equation~\eqref{eq:consistency_1} is $\alpha/(1-\gamma)$-optimal in the sparse MDP~\eqref{eq:mdp-reg-opt}, i.e.,~for each state $x\in\X$, we have
%
\begin{equation}
\label{eq:perf_tsallis_consistency}
V^\mu_{\text{sp}}(x) \geq V^*_{\text{sp}}(x) - \frac{\alpha}{1-\gamma}.
\end{equation}
%
\end{theorem}

This result indicates that for a fixed $\gamma$, as $\alpha$ decreases, 
a policy $\mu$ satisfying the one-step consistency equations
approaches the true optimal $\mu^*_{\text{sp}}$.
To connect the performance of $\mu$ to the original goal of maximizing expected return,
we present the following corollary, which 
is a direct consequence of Theorem~\ref{thm:perf_consistency} and the results reported in Section~\ref{subsec:emdp-sparse} 
on the performance of $\mu^*_{\text{sp}}$ in the original MDP.

\begin{corollary}
Any policy $\mu$ that satisfies the consistency equation~\eqref{eq:consistency_1} is $(\frac{3}{2}-\frac{1}{|\A|})\cdot\frac{\alpha}{1-\gamma}$-optimal in the original MDP~\eqref{eq:mdp-opt}, i.e.,~for each state $x\in\X$, we have
\begin{equation*}
V^*(x)-\left(\frac{3}{2}-\frac{1}{|\A|}\right)\cdot\frac{\alpha}{1-\gamma} \leq V^\mu(x) \leq V^*(x).
\end{equation*}
\end{corollary}



We now extend the (one-step) sparse consistency equation~\eqref{eq:consistency_1} to multiple steps.
For any state $x_0\in\X$ and sequence of actions $a_0,\ldots,a_{d-1}$, 
define the multi-step consistency equation for sparse MDPs as 


\vspace{-0.2in}
\begin{small}
\begin{align}
\label{eq:consistency_multi_step}
&V(x_0) = \mathbb{E}_{x_{1:d}\mid x_0,a_{0:d-1}}\Big[\gamma^dV(x_d) \\ 
&+ \sum_{t=0}^{d-1}\gamma^t\big(r(x_t,a_t) + \frac{\alpha}{2} - \alpha\mu(a_t|x_t) + \lambda(a_t|x_t) - \Lambda(x_t)\big)\Big], \nonumber
\end{align}
\end{small}
\vspace{-0.15in}

where $\lambda:\X\times\A\rightarrow\mathbb{R}_+$ and $\Lambda:\X\rightarrow\mathbb{R}_-$ are Lagrange multipliers, such that $\lambda(a|x)\cdot\mu(a|x)=0$ and $-\frac{\alpha}{2}\leq\Lambda(x)\leq 0$. We call this {\em multi-step sparse consistency} equation, the equivalent of Eq.~\ref{eq:multi-step-consistency-soft} in soft MDPs.


From Theorem~\ref{thm:consistency}, we can immediately show that multi-step sparse consistency is a necessary condition of optimality.

\begin{corollary}
\label{coro:nece_cond_multistep}
The optimal policy $\mu_{\text{sp}}^*$ and value function $V_{\text{sp}}^*$ of the sparse MDP~\eqref{eq:mdp-reg-opt} satisfy the multi-step consistency equation~\eqref{eq:consistency_multi_step}.
\end{corollary}

\begin{proof}
The proof follows directly from Theorem~\ref{thm:consistency}, by repeatedly applying the expression in~\eqref{eq:consistency_1} over the trajectory $\xi$, taking the expectation over the trajectory, and using telescopic cancellation of the value function of intermediate states.
\end{proof}

Conversely, followed from Theorem~\ref{thm:perf_consistency}, we prove the following result on the performance of any policy satisfying the mutli-step consistency equation. This is a novel result showing that solving the multi-step consistency equation
is indeed  \emph{sufficient} to guarantee near-optimality.

\begin{corollary}
\label{coro:suff_cond_multistep}
Any policy $\mu$ that satisfies the multi-step consistency equation~\eqref{eq:consistency_multi_step} is $\alpha/(1-\gamma)$-optimal in the sparse MDP~\eqref{eq:mdp-reg-opt}.
\end{corollary}

\begin{proof}
Consider the multi-step consistency equation~\eqref{eq:consistency_multi_step}. Since it is true for any initial state $x_0$ and sequence of actions $a_{0:d-1}$, unrolling it for another $d$ steps starting at state $x_d$ and using the action sequence $a_{d:2d-1}$ yields 

\vspace{-0.2in}
\begin{small}
\begin{align*}
&V(x_0) = \mathbb{E}_{x_{1:2d}|x_0,a_{0:2d-1}}\Big[\gamma^{2d}  V(x_{2d}) \\
&+\sum_{t=0}^{2d-1}\gamma^t\big(r(x_t,a_t) + \frac{\alpha}{2} - \alpha\mu(a_t|x_t) + \lambda(a_t|x_t) - \Lambda(x_t)\big)\Big].
\end{align*}
\end{small}
\vspace{-0.15in}

Note that this process can be repeated for an arbitrary number of times (say $k$ times), and also note that as $V$ is a bounded function, one has $\lim_{k\rightarrow\infty}\gamma^{kd}V(x_{kd})=0$. Therefore, by further unrolling, we obtain

\vspace{-0.2in}
\begin{small}
\begin{align*}
V(x_0) &= \mathbb{E}_{x_{1:\infty}|x_0,a_{0:\infty}}\Big[\sum_{t=0}^{\infty}\gamma^t\big(r(x_t,a_t) + \frac{\alpha}{2} - \alpha\mu(a_t|x_t) \\ 
&+ \lambda(a_t|x_t) - \Lambda(x_t)\big)\Big]. 
\end{align*}
\end{small}
\vspace{-0.15in}

Followed from the Banach fixed-point theorem \cite{Bertsekas96NP}, one can show that the solution pair $(V,\mu)$ is also a solution to the one-step consistency condition in \eqref{eq:consistency_1}, i.e., $ V(x)=r(x,a)+\frac{\alpha}{2}-\alpha \mu(a|x)+\lambda(a|x)-\Lambda(x)+\gamma\sum_{x'\in\X}P(x'|x,a) V(x')$, for any $x\in\X$ and $a\in\A$. Thus the $\alpha/(1-\gamma)$-optimality performance guarantee of $ \mu$ is implied by Theorem \ref{thm:perf_consistency}. 
\end{proof}

Equipped with the above results on the relationship between (near)-optimality and multi-step consistency in sparse MDPs, we are now ready to present our off-policy RL algorithms to solve the sparse MDP~\eqref{eq:mdp-reg-opt}.

%% file: pcl-sparse.tex

Similar to the PCL algorithm for soft MDPs, in sparse MDPs the multi-step consistency equation~\eqref{eq:consistency_multi_step} naturally leads to a path-wise algorithm for training a policy $\mu_\theta$ and value function $V_{\phi}$ parameterized by $\theta$ and $\phi$, as well as Lagrange multipliers $\Lambda_\rho$ and $\lambda_{\theta,\rho}$ parameterized by the auxiliary parameter $\rho$. To characterize the objective function of this algorithm, we first define the \emph{soft consistency error} for the $d$-step sub-trajectory $\xi$ as a function of $\theta$, $\rho$, and $\phi$,

\vspace{-0.2in}
\begin{small}
\begin{align*}
&J(\xi;\theta,\!\rho,\phi) = -V_\phi(x) + \gamma^d  V_\phi(x_d) \\
&+ \sum_{t=0}^{d-1}\gamma^j\big(r(x_t,a_t) + \frac{\alpha}{2} - \alpha \mu_\theta(a_t|x_t) + \lambda_{\theta,\rho}(a_t|x_t) - \Lambda_\rho(x_t)\big).
\end{align*}
\end{small}
\vspace{-0.1in}

The goal of our algorithm is to learn $V_\phi$, $\mu_\theta$, $\lambda_{\theta,\rho}$, and $\Lambda_\rho$, such that the expectation of $J(\xi;\theta,\rho,\phi)$ for any initial state $x_0$ and action sequence $a_{0:d-1}$ is as close to $0$ as possibles. Our sparse PCL algorithm minimizes the empirical objective function $\mathcal{J}_n(\theta, \rho, \phi)=\frac{1}{2}\!\sum_{\xi_i}J(\xi_i;\theta,\rho,\phi)^2$, which converges to $\mathcal J(\theta,\rho,\phi)=\mathbb E_{x_0,a_{0:d-1}}\left[\mathbb E[ J(\xi;\theta,\rho,\phi)^2\mid x_0,a_{0:d-1}]\right]$ as $i\rightarrow\infty$. By the Cauchy-Schwarz inequality, $\mathcal J(\theta, \rho, \phi)$ is a conservative surrogate of $\mathbb E\left[ J(\xi;\theta,\rho,\phi)\right]^2$, which represents the error of the multi-step consistency equation. This relationship justifies that the solution policy of the sparse PCL algorithm is near-optimal (see Corollary \ref{coro:suff_cond_multistep}). Moreover, the gradient of $J(\xi)$ w.r.t.~the parameters is as follows:


\vspace{-0.2in}
\begin{small}
\begin{align*}
\frac{\partial J(\xi)}{\partial\theta} &= J(\xi;\theta,\rho,\phi)\sum_{t=0}^{d-1}\gamma^t\nabla_\theta\big(\lambda_{\theta,\rho}(a_t|x_t) - \alpha \mu_\theta(a_t|x_t)\big), \\
\frac{\partial J(\xi)}{\partial\rho} &= J(\xi;\theta,\rho,\phi)\sum_{t=0}^{d-1}\gamma^t\nabla_\rho \big(\lambda_{\theta,\rho}(a_t|x_t) - \Lambda_\rho(x_t)\big), \\
\frac{\partial J(\xi)}{\partial\phi} &= J(\xi;\theta,\rho,\phi)\nabla_\phi \big( V_\phi(x_0) - \gamma^d V_\phi(x_d)\big).
\end{align*}
\end{small}
\vspace{-0.15in}

We may relate the sparse PCL algorithm to the standard actor-critic (AC) algorithm~\cite{konda2000actor,Sutton00PG}, where ${\partial J(\xi)}/{\partial\theta}$ and ${\partial J(\xi)}/{\partial\phi}$ correspond to the actor and critic updates, respectively. An advantage of sparse PCL over the standard AC is that it does not need the multi-time-scale update required by AC for convergence. 

While optimizing $J(\theta,\rho,\phi)$ minimizes the mean square of the soft consistency error, in order to satisfy the multi-step consistency in \eqref{eq:consistency_multi_step}, one still needs to impose the following constraints on Lagrange multipliers into the optimization problem: {\bf (i)} $-\frac{\alpha}{2}\leq\Lambda_\rho(x)\leq 0$, and {\bf (ii)}
$\lambda_{\theta,\rho}(a|x)\cdot\mu_\theta(a|x)=0$, $\forall x\in\X$, $\forall a\in\A$. One standard approach is to replace the above constraints with adding penalty functions \cite{bertsekas1999nonlinear} to the original objective function $\mathcal {J}_n$. Note that each penalty function is associated with a penalty parameter and there are $|\X|\cdot|\A|+2|\X|$ constraints. When $|\mathcal X|$ and $|\mathcal A|$ are large, tuning all the parameters becomes computationally expensive.
Another approach is to update the penalty parameters using gradient ascent methods \cite{bertsekas2014constrained}. This is equivalent to finding the saddle point of the Lagrangian function in the constrained optimization problem. However, the challenge is to balance the primal and dual updates in practice.

We hereby describe an alternative and a much simpler methodology to parameterize the Lagrange multipliers $\lambda_{\theta,\rho}(a|x)$ and $\Lambda_\rho(x)$, such that the aforementioned constraints are immediately satisfied. Although this method may impose extra restrictions to the representations of their function approximations, it avoids the difficulties of directly solving a constrained optimization problem. Specifically, to satisfy the constraint {\bf (i)}, one can parameterize $\Lambda_\rho$ with a multilayer perceptron network that has either an activation function of $-\alpha/2\cdot\sigma(\cdot)$ or $-\alpha/2\cdot(1+\tanh(\cdot))/2$ at its last layer. To satisfy constraint {\bf (ii)}, we consider the case when $\mu_\theta$ is written in form of $(f_\theta(x,a))^+$ for some function approximator $f_\theta$. 
This parameterization of $\mu_\theta$ is justified by the closed-form solution policy of the Tsallis entropy-regularized MDP problem in \eqref{eq:sparse-opt}.
Specifically,~\eqref{eq:sparse-opt} uses $f^*_{\text{sp}}(x,a)={Q^*_{\text{sp}}(x,a)}/{\alpha}-\mathcal G({Q^*_{\text{sp}}(x,\cdot)}/{\alpha})$. 
Now suppose that $\lambda_{\theta,\rho}$ is parameterized as follows: $\lambda_{\theta,\rho}(a|x)=\!\left(-f_\theta(x,a)\right)^+\cdot F_\rho(x,a)$, where $F_\rho:\X\times\A\rightarrow\mathbb R^+$ is an auxiliary function approximator. 
Then by the property $(x)^+\cdot (-x)^+=0$, constraint {\bf (ii)} is immediately satisfied. A pseudo-code of our sparse PCL algorithm can be found in Algorithm~\ref{alg:tsallis} in the Appendix~\ref{sec:appendix-proofs}.

\begin{paragraph}
{\bf Unified Sparse PCL} Note that the closed-form optimal policy $\mu^*_{\text{sp}}$ and value function $V^*_{\text{sp}}$ are both functions of the optimal state-action value function $Q^*_{\text{sp}}$. As in soft PCL, 
based on this observation one can also parameterize both policy and value function in sparse PCL (see Eq.~\ref{eq:sparse-opt}) with a single function approximator $Q_\psi(x,a)$. 
Although consistency does not imply optimality in sparse MDPs (as opposed to the case of soft MDPs), the justification of this parameterization comes from Corollary~\ref{coro:nece_cond_multistep}, where the unique optimal value function and optimal policy satisfy the consistency equation~\eqref{eq:consistency_multi_step}. From an actor-critic perspective, the significance of this 
is that both policy \emph{(actor)} and value function \emph{(critic)} can be updated simultaneously without affecting the convergence.  
Accordingly, the update rule for the model parameter $\psi$ takes the form 

\vspace{-0.2in}
\[
\small
\small
\begin{split}
\frac{\partial J(\xi)}{\partial\psi} &= J(\xi;\psi,\rho)\Big(\sum_{t=0}^{d-1}\gamma^t\nabla_\psi\big(\lambda_{\psi,\rho}(a_t|x_t)-\!\alpha \mu_\psi(a_t|x_t)\big) \\
&+\nabla_\psi V_\psi(x_0)-\gamma^d\nabla_\psi V_\psi(x_d)\Big).
\end{split}
\]
\vspace{-0.2in}
\end{paragraph}

%% file: experiment.tex
We demonstrate the effectiveness of the sparse PCL algorithm by comparing its performance with that of the soft PCL algorithm on a number of RL environments available in the OpenAI Gym~\cite{gym} environment.

\begin{figure*}[h]
\begin{center}
  \setlength\tabcolsep{3pt}
  \renewcommand{\arraystretch}{0.2}
  \begin{tabular}{ccccc}

    \multirow{4}{*}{\rotatebox[origin=c]{90}{\tiny Copy\hspace{1.0cm}}}
    & \tiny $|\A|=20$  &  \tiny $|\A|=40$  &  \tiny $|\A|=80$  &  \tiny $|\A|=160$ \\
    &
    \includegraphics[width=0.32\columnwidth]{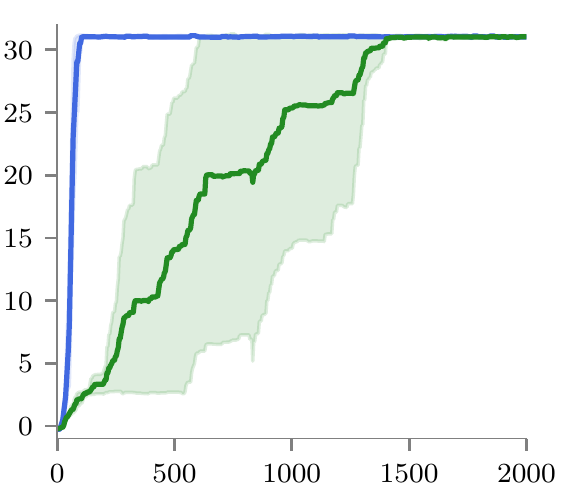} &
    \includegraphics[width=0.32\columnwidth]{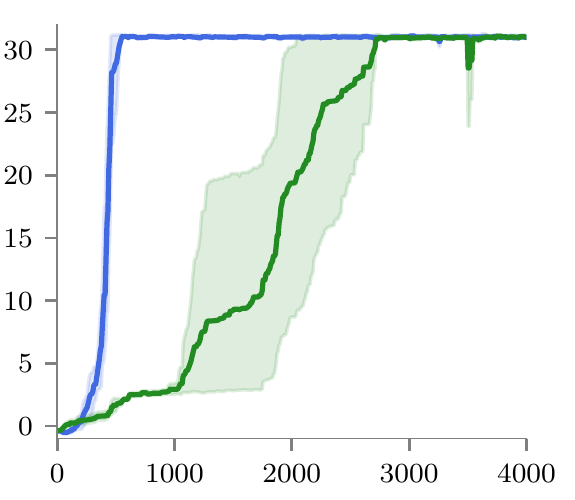} &
    \includegraphics[width=0.32\columnwidth]{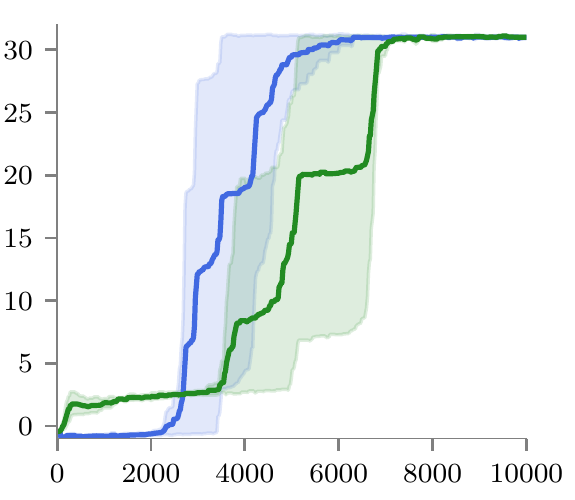} &
    \includegraphics[width=0.32\columnwidth]{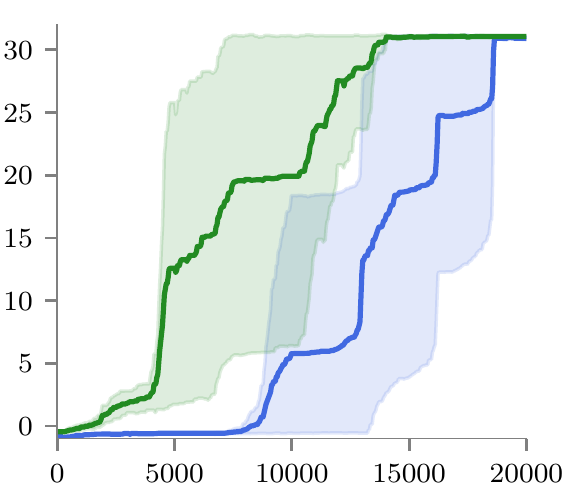} \\

    \multirow{2}{*}{\rotatebox[origin=c]{90}{\tiny DuplicatedInput\hspace{0.6cm}}}
    & \tiny $|\A|=20$  &  \tiny $|\A|=80$  &  \tiny $|\A|=160$  &  \tiny $|\A|=320$ \\
    &
    \includegraphics[width=0.32\columnwidth]{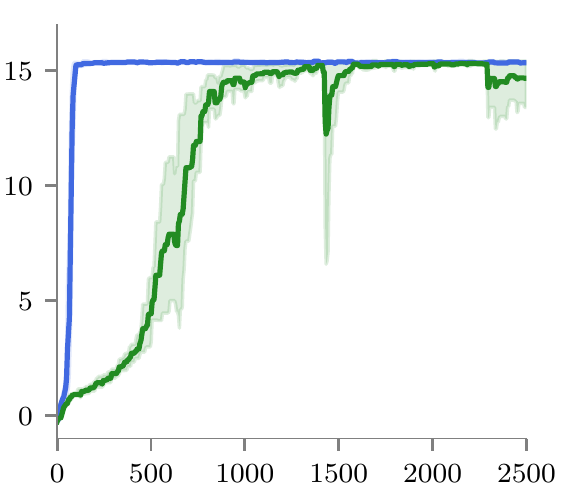} &
    \includegraphics[width=0.32\columnwidth]{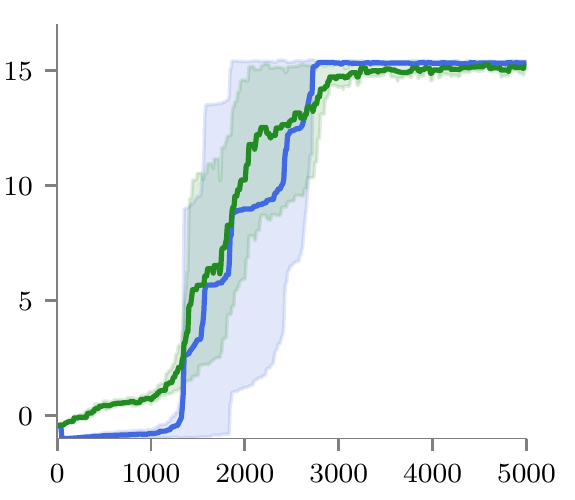} &
    \includegraphics[width=0.32\columnwidth]{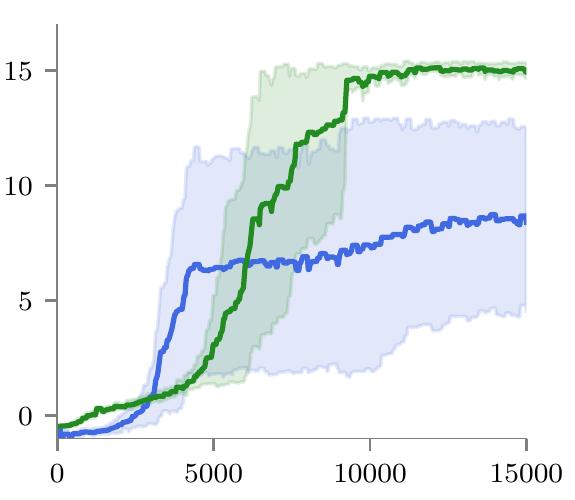} &
    \includegraphics[width=0.32\columnwidth]{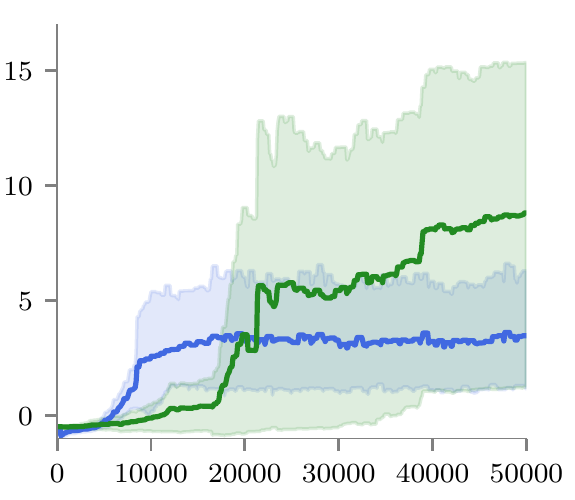} \\

    \multirow{2}{*}{\rotatebox[origin=c]{90}{\tiny RepeatCopy\hspace{0.8cm}}}
    & \tiny $|\A|=20$  &  \tiny $|\A|=40$  &  \tiny $|\A|=80$  &  \tiny $|\A|=160$ \\
    &
    \includegraphics[width=0.32\columnwidth]{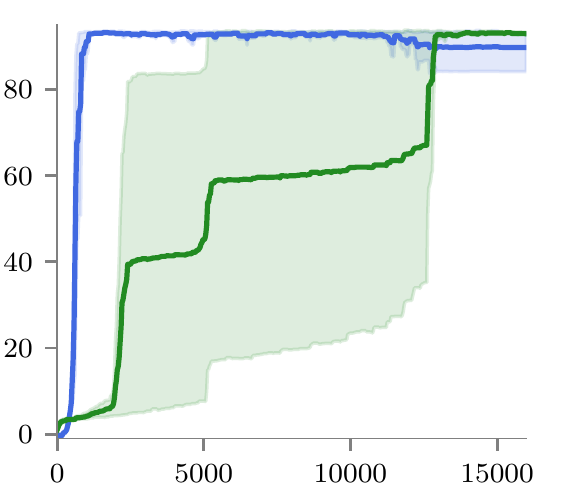} &
    \includegraphics[width=0.32\columnwidth]{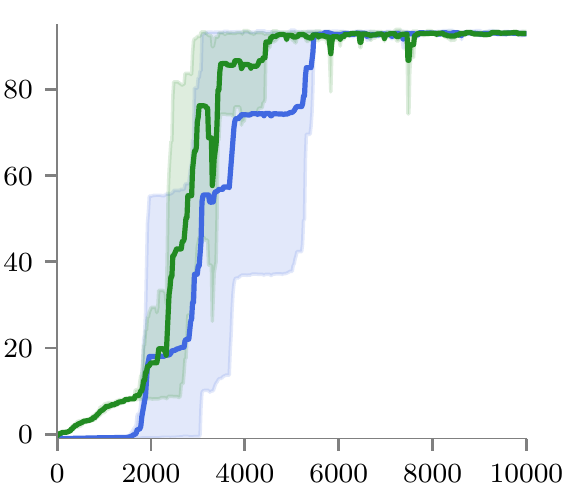} &
    \includegraphics[width=0.32\columnwidth]{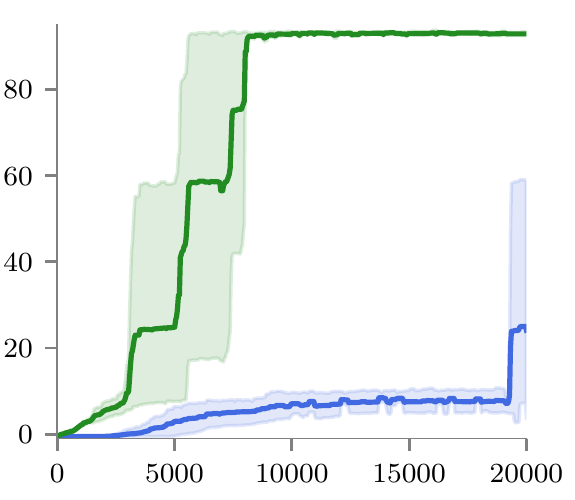} &
    \includegraphics[width=0.32\columnwidth]{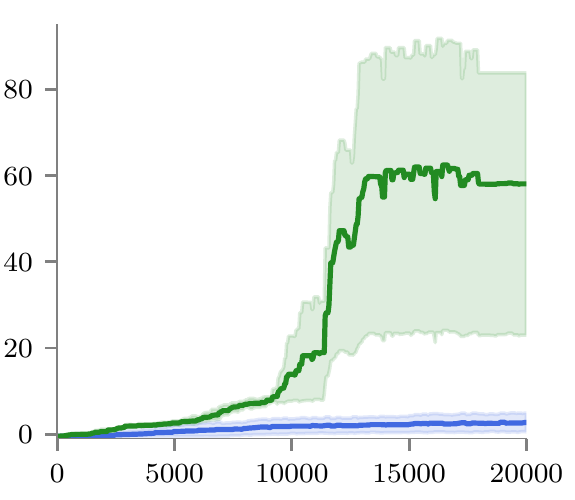} \\

    \multirow{2}{*}{\rotatebox[origin=c]{90}{\tiny Reverse\hspace{0.9cm}}}
    & \tiny $|\A|=16$  &  \tiny $|\A|=32$  &  \tiny $|\A|=64$  &  \tiny $|\A|=128$ \\
    &
    \includegraphics[width=0.32\columnwidth]{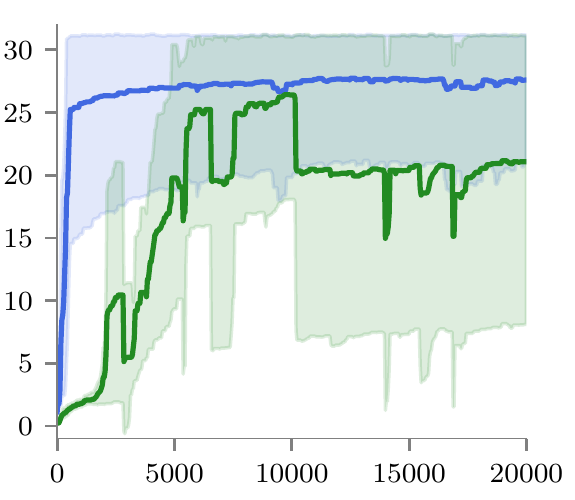} &
    \includegraphics[width=0.32\columnwidth]{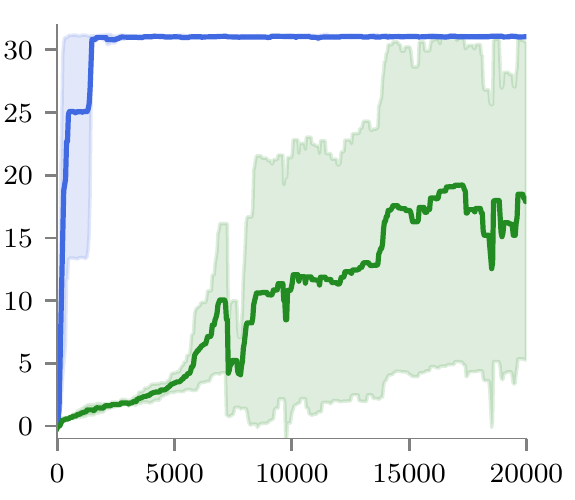} &
    \includegraphics[width=0.32\columnwidth]{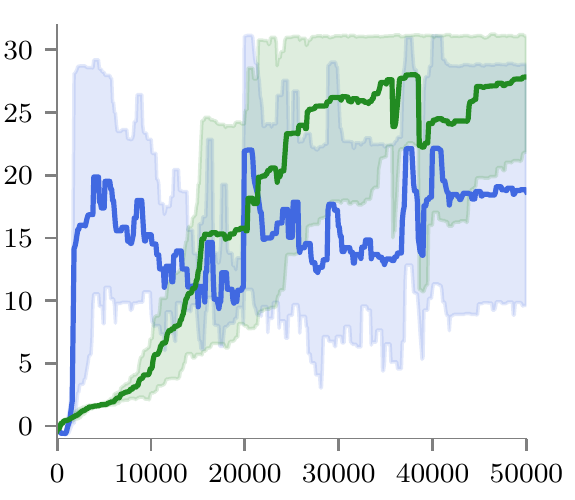} &
    \includegraphics[width=0.32\columnwidth]{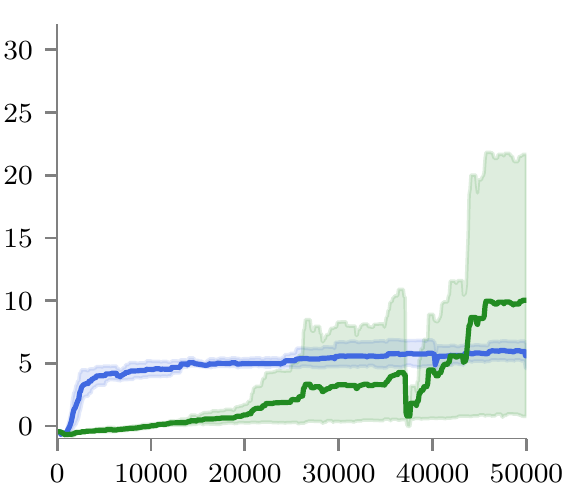} \\

    \multicolumn{5}{c}{\includegraphics[width=0.5\columnwidth]{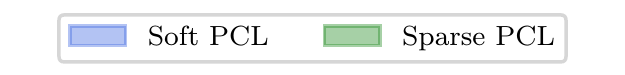}}
  \end{tabular}
\end{center}
\caption{
Results of the average reward from sparse PCL and 
standard soft PCL during training. Here each row corresponds to a specific algorithmic task.
For each particular task, the action space is increased from left to right across the rows,
corresponding to an increase in difficulty.
We observe that soft PCL returns a better solution when the action space is small, but its performance degrades quickly as the size of action space grows.
On the other hand, 
sparse PCL is not only able to learn good policies in tasks with small action spaces, 
but, unlike soft PCL, also successfully learns high-reward policies in the higher-dimension variants.
See the appendix for additional results.
}
\label{fig:results}
\end{figure*}

\subsection{Discrete Control}
Here we compare the performance of these two algorithms on the following standard algorithmic tasks: {\bf 1)} Copy, {\bf 2)} DuplicatedInput, {\bf 3)} RepeatCopy, {\bf 4)} Reverse, and {\bf 5)} ReversedAddition (see appendix for more details).
Each task can be viewed as a grid environment, where each cell stores a single character
from a finite vocabulary $\mathcal{V}$.  
An agent moves on the grid of the environment and writes to output.
At each time step the agent observes
the character of the single cell in which it is located.
After observing the character, the agent must take an action of the form 
$(m, w, c)$, where $m$ 
determines the agent's move to an adjacent cell,
(in 1D environments, $m\in\{\text{left}, \text{right}\}$;
in 2D environments, $m\in\{\text{left}, \text{right}, \text{up}, \text{down}\}$),
$w\in\{0, 1\}$ determines whether the agent writes to output or not,
and $c\in \mathcal{V}$ determines the character that the agent writes if $w=1$ (otherwise $c=\emptyset$).
Based on this problem setting, the action space $\A$ has size $|\A|=\Theta(|\mathcal{V}|)$. 
Accordingly, the difficulty of these tasks grows with the size of the vocabulary. 
To illustrate the effectiveness of Tsallis entropy-regularized MDPs in problems with large action space, 
we evaluate these two PCL algorithms on $4$ different choices of $|\mathcal V|$.

In each task, the agent has a different goal. In Copy,
the environment is a 1D sequence of characters and the 
agent aims to copy the sequence to output. 
In DuplicatedInput, the
environment is a 1D sequence of duplicated characters and the
agent needs to write the de-duplicated sequence to output.
In RepeatCopy, the environment is a 1D sequence of characters
in which the agent must copy in forward order, reverse order, and finally
forward order again. 
In Reverse, the environment is a 1D sequence of characters in which the
agent must copy to output in reverse order.
Finally, in ReversedAddition, the environment is a $2\times n$ grid of digits
representing two numbers in base-$|\mathcal{V}|$ that the agent needs to sum.
In each task the agent receives a reward of $1$ for each correctly
output character.  The episode is terminated either when the task is completed, or when the agent outputs an incorrect character.

We follow a similar experimental procedure as in~\citet{pcl}, where
the functions $V$, $\mu$, $\lambda$, $\Lambda$ in the consistency equations are parameterized with a recurrent neural
network with multiple heads.  
For each task and each PCL algorithm, we perform a hyper-parameter search to find the optimal
regularization weight $\alpha$, and the corresponding training curves for average reward are shown in Figure \ref{fig:results}. To increase the statistical significance of these experiments, we also train these policies on $5$ different Monte Carlo trials (Notice that these environments are inherently deterministic, therefore no additional Monte Carlo evaluation is needed.). Details of the experimental setup and extra numerical results are included in the Appendix.

For each task we evaluated sparse PCL compared to the original soft PCL
on a suite of variants which successively increase the vocabulary size.  
For low vocabulary sizes soft PCL achieves better results.
This suggests that Shannon entropy encourages better exploration in
small action spaces.  Indeed, in such regimes, a greater proportion
of the total actions are useful to explore, and exploration is not as costly.
Therefore, the decreased exploration of the Tsallis entropy may outweigh
its asymptotic benefits.
The sub-optimality bounds presented in this paper support this behavior:
when $|\A|$ is small, 
$\alpha_{\text{soft PCL}}\log (|\A|)\leq 3\alpha_{\text{sparse PCL}}/2$.

As we increase the vocabulary size (and thus the action space), the picture changes. 
We see that the advantage of soft PCL over sparse PCL decreases until eventually the
order is reversed and sparse PCL begins to show a significant improvement
over the standard soft PCL.
This supports our original hypothesis.
In large action spaces, the tendency of soft PCL to assign
a non-zero probability to many sub-optimal actions over-emphasizes exploration
and is detrimental to the final reward performance.
On the other hand, sparse PCL
is able to handle exploration
in large action spaces properly. 
These empirical results provide evidence for this unique advantage of sparse PCL.

\subsection{Continuous Control}
We further evaluate the two PCL algorithms on HalfCheetah, a continuous control problem in the OpenAI gym.
The environment consists of a $6-$dimensional action space, where each dimension
corresponds to a torque of $[-1, 1]$.  Here we discretize each continuous action with either one of the following grids: $\{-1, 0, 1\}$ and $\{-1, -0.5, 0, 0.5, 1\}$.  
Even though the resolution of these discretization grids is coarse,
the corresponding action spaces are quite large, with sizes of
$3^6=729$ and $5^6=15625$, respectively.

We present the results of sparse PCL and soft PCL on these discretized
problems in Figure~\ref{fig:continuous}.
Similar to the observations in the algorithmic tasks, here the policy learned by sparse PCL performs much better than that of soft PCL. Specifically sparse PCL achieves higher average reward and is able to learn much faster.
To better visualize the learning progress of these two PCL algorithms 
in these problems, at each training step we also compare the average probability of the most-likely actions across all time-steps from the on-policy trajectory.\footnote{Specifically in each iteration we collect a single on-policy trajectory of $1000$ steps. Therefore this metric is an average over $1000$ samples of (greedy) action probabilities.}
Clearly, sparse PCL quickly converges to a near-deterministic policy, 
while the policy generated by soft PCL still allocates significant probability masses to non-optimal actions (as the average probability of most-likely actions barely ever exceeds $0.75$).
In environments like HalfCheetah, where
the trajectory has a long horizon ($1000$ steps), the soft-max policy will in general suffer 
because it chooses a large number of sub-optimal actions in each episode for exploration.

Comparing with the performance of other continuous RL algorithms such as deterministic policy gradient (DPG)
\cite{silver2014deterministic}, we found that the policy generated by sparse PCL is sub-optimal. This is mainly due to the coarse discretization of the action space. Our main purpose here is to demonstrate the fast and improved convergence to deterministic policies in sparse PCL, compared to soft PCL. Further evaluation of sparse PCL will be left to future work.

\begin{figure}[t]
\vspace{-0.1in}
\begin{center}
  \setlength\tabcolsep{1.5pt}
  \renewcommand{\arraystretch}{0.0}
 
  \begin{tabular}{ccc}

    & $|\A|=3^6$  &  $|\A|=5^6$  \\
    \rotatebox{90}{\hspace{1cm}Reward}
    &
    \includegraphics[width=0.42\columnwidth]{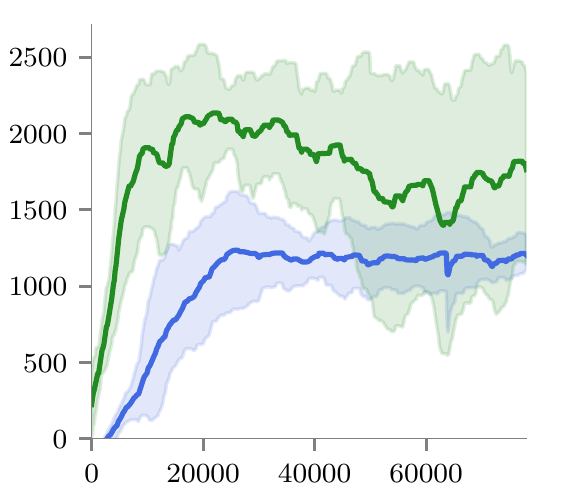} &
    \includegraphics[width=0.42\columnwidth]{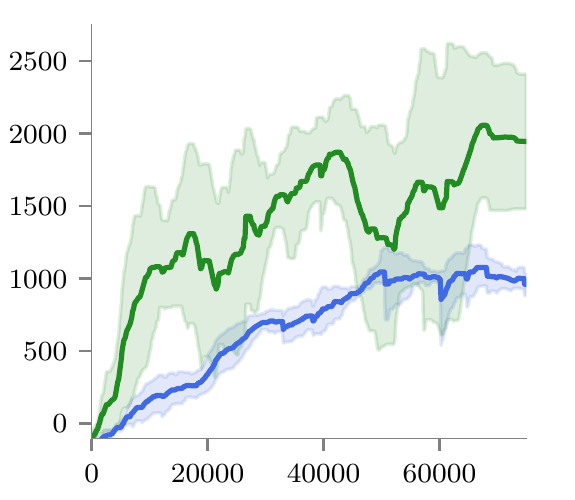} \\

    \rotatebox{90}{\pbox{3cm}{\hspace{0.7cm}Most Likely \\ $~~~~~~~~$ Probability}}
    &
    \includegraphics[width=0.42\columnwidth]{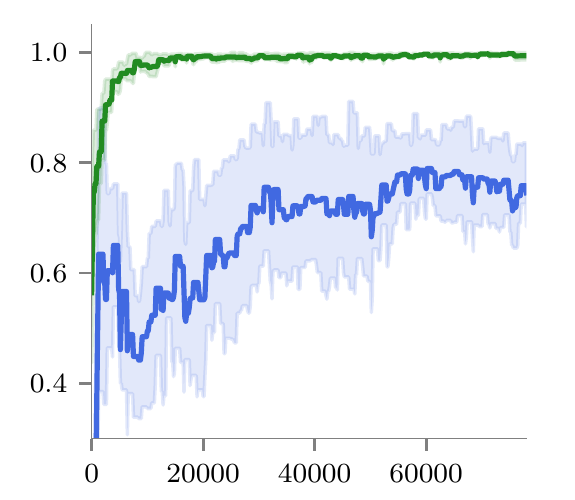} &
    \includegraphics[width=0.42\columnwidth]{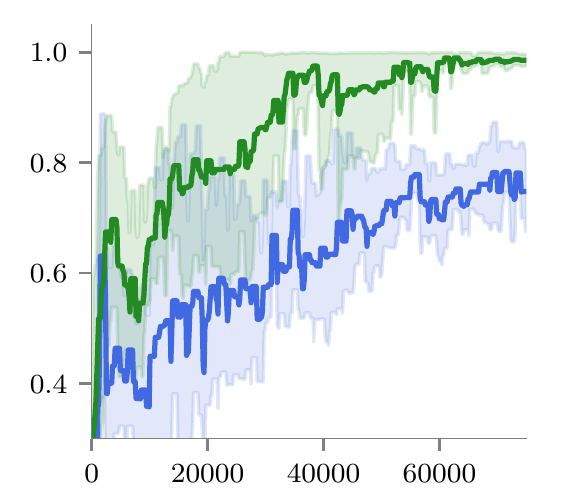} \\

    \multicolumn{3}{c}{\includegraphics[width=0.5\columnwidth]{rewards_-1_-1}}
  \end{tabular}
\end{center}
\caption{
Results of sparse PCL and soft PCL in HalfCheetah with discretized actions.
The top figure shows the average reward over $5$ random runs during training, with best hyper-parameters.
On the bottom we plot the average probability of the most-likely actions
during training. The bottom figure illustrates the fast convergence of sparse PCL to a near-deterministic policy.
}
\label{fig:continuous}
\end{figure}

\vspace{-0.1in}

%% file: conclusion.tex
In this work we studied the sparse entropy-regularized problem in RL, 
whose optimal policy has non-zero probability for only a small number of actions. 
Similar to the work by~\citet{pcl}, we derived a relationship between (near-)optimality and consistency for this problem. 
Furthermore, by leveraging the properties of the consistency equation,
we proposed a class of sparse path consistency learning (sparse PCL) algorithms that 
are applicable to both on-policy and off-policy data and can learn from multi-step
trajectories. 
We found that the theoretical advantages of sparse PCL correspond to empirical advantages
as well.
For tasks with a large number of actions, 
we find significant improvement in final performance and amount of time needed to reach that performance
by using sparse PCL compared to the original soft PCL.

Future work includes {\bf 1)} extending the sparse PCL algorithm to the more general class of Tsallis entropy, {\bf 2)} investigating the possibility of combining sparse PCL and  path following algorithms such as TRPO \cite{Schulman15TR}, and {\bf 3)} comparing the performance of sparse PCL with other deterministic policy gradient algorithms, such as DPG \cite{silver2014deterministic} in the continuous domain.

%% file: proofs.tex
Consider the Bellman operator for the entropy-regularized MDP with Tsallis entropy:
\[
(\mathcal{T}_{\text{sp}}f)(x) = \alpha \cdot \text{spmax}\Big(\big[r(x,\cdot)+\gamma\sum_{x'}P(x'|x,\cdot)f(x')\big]/\alpha\Big).
\]
We first have the following technical result about its properties.
\begin{proposition}
The sparse-max Bellman operator $\mathcal T_{\text{sp}}$ has the following properties: (i) Translation: $(\mathcal T_{\text{sp}} (V+\beta))(x)=\mathbf (\mathcal T_{\text{sp}} V)(x)+\gamma\beta$; (ii) $\gamma$-contraction: $\|\mathbf (\mathcal T_{\text{sp}} V_1)-\mathbf (\mathcal T_{\text{sp}} V_2)\|_\infty\leq \gamma \|V_1-V_2\|_{\infty}$; (iii) Monotonicity: $\mathbf (\mathcal T_{\text{sp}} V_1)(x)\leq \mathbf (\mathcal T_{\text{sp}} V_2)(x)$ for any value functions $V_1,V_2:\X\rightarrow\mathbb R$ such that $V_1(x)\leq V_2(x)$, $\forall x\in\X$.
\end{proposition}
The detailed proof of this proposition can be found in \citet{Lee18SM}. Using these results, the Banach fixed point theorem shows that there exists a unique solution for the following fixed point equation: $V(x)=\mathbf (\mathcal T_{\text{sp}} V)(x)$, $\forall x\in\X$, and this solution is equal to the optimal value function $V^*_{\text{sp}}(x)$. Analogous to the arguments in standard MDPs, in this case the optimal value function can also be computed using dynamic programming methods such as value iteration.

Before proving the main results, notice that by using analogous arguments of the complementary-slackness property in KKT conditions, the second and the third consistency equation in \eqref{eq:consistency_1} is equivalent to the following condition:
\begin{equation}\label{eq:consistency_2}
\begin{split}
&r(x,a)+\gamma\sum_{x'\in\X}P(x'|x,a)V(x')+\frac{\alpha}{2}-\alpha\mu(a|x)-V(x)= \Lambda(x),\,\forall x\in\X,\,\, \forall a\in\A_{\mu}(x),\\
&r(x,a)+\gamma\sum_{x'\in\X}P(x'|x,a)V(x')+\frac{\alpha}{2}-\alpha\mu(a|x)-V(x)\leq \Lambda(x),\,\forall x\in\X,\,\, \forall a\not\in\A_{\mu}(x),
\end{split}
\end{equation}
where $\A_{\mu}(x)=\{a\in\A:\mu(a|x)>0\}$ represents the set of actions that have non-zero probabilities w.r.t policy $\mu$. 

\begin{theorem}
The pair of optimal value function and optimal policy $(V_{\text{sp}}^*,\mu_{\text{sp}}^*)$ of the MDP problem in \eqref{eq:mdp-reg-opt} satisfies the consistency equation in \eqref{eq:consistency_1}.
\end{theorem}
\begin{proof}
Recall that the optimal state-action value function is given by
\[
Q_{\text{sp}}^*(x,a)=r(x,a)+\gamma\sum_{x'\in\X}P(x'|x,a)V_{\text{sp}}^*(x').
\]
According to Bellman's optimality, the optimal value function satisfies the following equality:
\begin{equation}\label{eq:Bellman_tsallis}
V_{\text{sp}}^*(x)=\max_{\mu\in\Delta_x}\sum_{a\in\A}\mu(a|x)\left[Q_{\text{sp}}^*(x,a)+\frac{\alpha}{2}(1-\mu(a|x))\right],
\end{equation}
at any state $x\in\X$, where $\mu_{\text{sp}}^*$ is the corresponding maximizer. By the KKT condition, we have that
\[
Q_{\text{sp}}^*(x,a)+\frac{\alpha}{2}(1-\mu^*_{\text{sp}}(a|x))+\lambda^*_{\text{sp}}(a|x)=\Lambda^*_{\text{sp}}(x)+\frac{\alpha}{2}\mu^*_{\text{sp}}(a|x), 
\]
for any $x\in\X$ and any $a\in\A$, where $\Lambda^*_{\text{sp}}$
is the Lagrange multiplier that corresponds to equality constraint 
$
\sum_{a\in\A}\mu(a|x)=1
$, and $\lambda^*_{\text{sp}}\geq 0$ is the Lagrange multiplier that corresponds to inequality constraint $\mu(a|x)\geq 0$ such that 
\[
\lambda^*_{\text{sp}}(a|x)\cdot\mu^*_{\text{sp}}(a|x)=0,\,\,\forall x\in\X, \,\,\forall a\in\A.
\]
Recall from the definition of optimal state-action value function $Q_{\text{sp}}^*$ and the definition of the optimal policy $\mu_{\text{sp}}^*$, one has that $\mathcal A_{\mu_{\text{sp}}^*}(x)=\mathcal S(Q_{\text{sp}}^*(x,\cdot))$.
This condition further implies
\[
\Lambda^*_{\text{sp}}(x)=Q_{\text{sp}}^*(x,a)+\frac{\alpha}{2}(1-2\mu^*_{\text{sp}}(a|x)),\,\,\forall x\in\X, a\in\mathcal S(Q_{\text{sp}}^*(x,\cdot)).
\]
Substituting the equality in \eqref{eq:Bellman_tsallis} to this KKT condition, and noticing that $0\leq\sum_{a\in\A}\mu_{\text{sp}}^*(a|x))^2\leq 1$, the KKT condition implies that
\[
 \Lambda^*_{\text{sp}}(x)+\frac{\alpha}{2}\geq V_{\text{sp}}^*(x)=\Lambda^*_{\text{sp}}(x)+\frac{\alpha}{2}\sum_{a\in\A}\mu^*_{\text{sp}}(a|x)\mu^*_{\text{sp}}(a|x)\geq \Lambda^*_{\text{sp}}(x),
\]
which further implies that
\[
-\frac{\alpha}{2}\leq \Lambda^*_{\text{sp}}(x)-V_{\text{sp}}^*(x)\leq 0,\,\,\forall x\in\X.
\]
Therefore, by defining $\Lambda(x)=\Lambda^*_{\text{sp}}(x)-V_{\text{sp}}^*(x)$, and $\lambda(a|x)=\lambda^*_{\text{sp}}(a|x)$, one immediately has that $-\frac{\alpha}{2}\leq\Lambda(x)\leq 0$, $\forall x\in\X$. Using this construction, one further has the following expression for any $x\in\X$ and any $a\in\mathcal S(Q^*_{\text{sp}}(x,\cdot))$:
\[
\Lambda(x)=Q_{\text{sp}}^*(x,a)+\frac{\alpha}{2}-\alpha\mu^*_{\text{sp}}(a|x)-V_{\text{sp}}^*(x),
\]
which proves consistency, based on the equivalence condition in \eqref{eq:consistency_2}.
\end{proof}

\begin{theorem}
The solution policy $\mu$ of the consistency equation in \eqref{eq:consistency_1} is $\alpha/(1-\gamma)$-optimal w.r.t. the sparse MDP problem in \eqref{eq:mdp-reg-opt}. That is,
\begin{small}
\begin{equation}
\mathbb E\left[\sum_{t=0}^{\infty}\gamma^t\left[ R_t+\frac{\alpha}{2}\left(1- \mu(a_t|x_t)\right)\right]\mid x_0=x,\mu,P\right]\geq V^*_{\text{sp}}(x)-\frac{\alpha}{(1-\gamma)}.
\end{equation}
\end{small}
\end{theorem}
\begin{proof}
To proof the sub-optimality performance bound given in this theorem, we first study the expression of $\mathbf (\mathcal T_{\text{sp}} V)$, where $\mathcal T_{\text{sp}}$ is the Bellman operator of the Tsallis entropy-regularized MDP problem in \eqref{eq:mdp-reg-opt}. Let 
\[
\bar Q(x,a)=r(x,a)+\gamma \sum_{x'\in\X}P(x'|x,a)V(x')
\]
be the corresponding state-action value function. Using the definition from \eqref{eq:mdp-reg-opt}, one has the following expression:
\[
\begin{split}
\mathbf (\mathcal T_{\text{sp}} V)(x)&=\alpha\cdot\text{spmax}\left(\frac{1}{\alpha}\cdot\left\{r(x,a)+\gamma\sum_{x'\in\X}P(x'|x,a)V(x')\right\}_{a\in\A}\right)\\
&=\alpha\cdot\text{spmax}\left(\frac{\bar Q(x,\cdot)}{\alpha}\right).
\end{split}
\]
Furthermore, by exploiting the structure of the sparse-max formulation of an arbitrary value function, one also has the following chain of equalities/inequalities:
\[
\begin{split}
\alpha\cdot\text{spmax}\left(\frac{\bar Q(x,\cdot)}{\alpha}\right)=&\max_{\mu\in\Delta_x}\sum_{a\in\A}\mu(a|x)\cdot\left(\bar Q(x,a)+\frac{\alpha}{2}(1-\mu(a|x))\right)\\
=&\sum_{a\in\A}\mu(a|x)\cdot\left(\bar Q(x,a)+\frac{\alpha}{2}(1-\mu(a|x))\right)\\
= &\sum_{a\in\A}\mu(a|x)\left(\bar Q(x)+\frac{\alpha}{2}-\alpha\mu(a|x)\right)+\frac{\alpha}{2}\sum_{a\in\A}\mu(a|x)^2\\
\leq & V(x)+\frac{\alpha}{2}.
\end{split}
\]
The first equality follows from the fact that $\alpha\cdot\text{spmax}({\bar Q(x,\cdot)}/{\alpha})$ is a closed form solution of the optimization problem 
\[
\max_{\mu\in\Delta_x} \sum_{a\in\A}\mu(a|x)\left(\bar Q(x,a)-\alpha \mathbb H_\mu(x,a)\right),
\]
when $\mathbb H_\mu$ is the Tsallis entropy.
The second equality follows from the fact that if $(V,\mu)$ satisfies the consistency equation, then there exists a Lagrange multiplier $\Lambda^*(x)=\Lambda(x)+V(x)$, $\forall x\in\X$ such that the following KKT conditions hold:
\[
\begin{split}
&\bar Q(x,a)+\frac{\alpha}{2}-\alpha\mu(a|x)=\Lambda^*(x),\,\,\forall x\in\X,\,\,\forall a\in\A_{\mu}(x),\\
&\bar Q(x,a)+\frac{\alpha}{2}-\alpha\mu(a|x)\leq\Lambda^*(x),\,\,\forall x\in\X,\,\,\forall a\not\in\A_{\mu}(x),\\
&\sum_{a}\mu(a|x)=1,\,\mu(a|x)\geq 0,\,\,\forall x\in\X, \,\,\forall a\in\A,
\end{split}
\]
which further implies that $\mu$ is the maximizer of the inner optimization problem.
The third equality follows from arithmetic manipulations, and the first inequality follows from the consistency equation in \eqref{eq:consistency_2}, i.e., for any $x\in\X$ and any $a\in\A_{\mu}(x)$, there exists $\Lambda(x)\in[-\frac{\alpha}{2},0]$ such that:
\[
0\geq\Lambda(x)=\bar Q(x,a)+\frac{\alpha}{2}-\alpha\mu(a|x)-V(x)
\iff \bar Q(x,a)+\frac{\alpha}{2}-\alpha\mu(a|x)\leq V(x).
\]
Therefore combining all these arguments, one concludes that the following Bellman inequality holds:
\begin{equation}\label{eq:Bellman_ineq}
\mathbf (\mathcal T_{\text{sp}} V)(x)\leq V(x)+\frac{\alpha}{2},\,\,\forall x\in\X.
\end{equation}

Now recall that the $\gamma-$contraction property (w.r.t. the $\infty-$norm) of the Bellman operator $\mathcal T_{\text{sp}}$ . By the Banach fixed-point theorem, this property implies that there exists a unique fixed point solution $V^*_{\text{sp}}$ to equation $V(x)=\mathbf (\mathcal T_{\text{sp}} V)(x)$, for all $x\in\X$, and it is the limit point (over all $x\in\X$) of the converging iterative sequence $\lim_{n\rightarrow\infty}(\mathcal T_{\text{sp}}^nV_0)(x)$ for any initial value function $V_0$.
Also recall that the translation property of this Bellman operator, i.e., for any constant $K$,
$
(\mathcal T_{\text{sp}}(V+K))=\mathbf (\mathcal T_{\text{sp}} V)+\gamma K
$.
Therefore, by repeatedly applying the Bellman operator to both sides of the inequality in \eqref{eq:Bellman_ineq}, and by using the above properties of a Bellman operator, one can show that
\begin{equation}\label{eq:result_1}
 V^*_{\text{sp}}(x)=\lim_{n\rightarrow\infty}(\mathcal T_{\text{sp}}^nV)(x)\leq\sum_{n=0}^{\infty}-\gamma^t \cdot\frac{\alpha}{2}+V(x)=-\frac{\alpha}{2}\cdot\frac{1}{1-\gamma}+V(x),\,\,\forall x\in\X.
\end{equation}

Furthermore, consider the consistency equation in \eqref{eq:consistency_1}, i.e., there exists a function $\Lambda(x)\in[0,\frac{\alpha}{2}]$ such that for any $x\in\X$ and any $a\in\A_{\mu}(x)$,
\[
-\frac{\alpha}{2}\leq\Lambda(x)=\bar Q(x,a)+\frac{\alpha}{2}-\alpha\mu(a|x)-V(x)\iff V(x)\leq\bar Q(x,a)+\alpha-\alpha\mu(a|x).
\]
By multiplying $\mu(a|x)$ on both sides of this inequality and summing over $a\in\A$, the above expression implies
\begin{equation}\label{eq:bellman_ineq_bar_mu}
\begin{split}
V(x)&\leq\sum_{a\in\A}\mu(a|x)\left(\bar Q(x,a)+\alpha-\alpha\mu(a|x)\right)\\
&\leq \sum_{a\in\A}\mu(a|x)\left(\bar Q(x,a)+\frac{\alpha}{2}\left(1-\mu(a|x)\right)\right)+\frac{\alpha}{2}\sum_{a\in\A}\mu(a|x)\left(1-\mu(a|x)\right)\\
&=\sum_{a\in\A}\mu(a|x)\left(r(x,a)+\gamma\sum_{x'\in\X}P(x'|x,a)V(x')+\frac{\alpha}{2}\left(1-\mu(a|x)\right)\right)+\frac{\alpha}{2}.
\end{split}
\end{equation}
Therefore, equipped with the $\gamma-$contraction property of the following Bellmen operator:
\[
(\mathcal T_{\mu} V)(x)=\sum_{a\in\A}\mu(a|x)\left(r(x,a)+\frac{\alpha}{2}\left(1-\mu(a|x)\right)+\gamma\sum_{x'\in\X}P(x'|x,a)V(x')\right)
\]
and the Banach fixed-point theorem, for any initial value function $V_0$, one can deduce the following expression:
\[
\lim_{n\rightarrow\infty}\mathbf T_{\mu}[V_0]^n(x)=\mathbb E\left[\sum_{t=0}^{\infty}\gamma^t\left[ r(x_t,a_t)+\frac{\alpha}{2}\left(1- \mu(a_t|x_t)\right)\right]\mid \mu,x_0=x\right].
\]
Using the translation property of the Bellman operator $(\mathcal T_{\mu} V)$ and repeatedly applying this Bellman operator to both sides of \eqref{eq:bellman_ineq_bar_mu}, one obtains the following inequality for any $x\in\X$:
\begin{equation}\label{eq:result_2}
\begin{split}
V(x)\leq& \lim_{n\rightarrow\infty}(\mathcal T_{\mu} V)^n(x)+\sum_{n=0}^{\infty}\frac{\alpha}{2}\gamma^t\\
=&\mathbb E\left[\sum_{t=0}^{\infty}\gamma^t\left[ r(x_t,a_t)+\frac{\alpha}{2}\left(1- \mu(a_t|x_t)\right)\right]\mid \mu,x_0=x\right]+\frac{\alpha}{2}\cdot\frac{1}{1-\gamma}.
\end{split}
\end{equation}

Therefore, by combining the results in \eqref{eq:result_1} and in \eqref{eq:result_2}, one completes the proof of this theorem.
\end{proof}

\begin{algorithm}[h]
\caption{Sparse Path Consistency Learning}
\label{alg:tsallis}    

\begin{algorithmic}
\STATE {\bfseries Input:} 
Environment $ENV$, 
learning rate $\eta$, discount factor $\gamma$,
regularization $\alpha$,
rollout $d$, number of steps $N$, replay buffer capacity $B$, prioritized
replay hyper-parameter $\alpha$. Parameterizations of $\Lambda$ and $\lambda$ follow from the descriptions in Section \ref{sec:pcl-sparse}.

\FUNCTION{Gradients($x_{0:T}, R_{0:T-1}, a_{0:T-1}$)}
\STATE Compute $C(t) = -\bar V_{\phi}(x_t) + \gamma^d \bar V_{\phi}(x_{t+d}) + \sum_{j=0}^{d-1} \gamma^j (R_j + \alpha/2 - \alpha\bar\mu_\theta(a_j|x_j) + \lambda_\theta(a_j|x_j) - \Lambda_\rho(x_j))$ for $t<T$, padding with zeros as necessary.
\STATE Compute $\Delta\theta = \sum_{t=0}^{T-1} C(t) \nabla_\theta C(t)$.
\STATE Compute $\Delta\phi = \sum_{t=0}^{T-1} C(t) \nabla_\phi C(t)$.
\STATE Compute $\Delta\rho = \sum_{t=0}^{T-1} C(t) \nabla_\rho C(t)$.
\STATE \emph{Return} $\Delta\theta, \Delta\phi, \Delta\rho$
\ENDFUNCTION

\STATE Initialize $\theta,\phi,\rho$.
\STATE Initialize empty replay buffer $RB(\alpha)$.
\FOR{$i=0$ {\bfseries to} $N-1$}
\STATE Sample $x_{0:T}, a_{0:T-1}\sim\bar\mu_\theta$ on $ENV$, yielding reward $R_{0:T-1}$.
\STATE $\Delta\theta, \Delta\phi, \Delta\rho = \text{Gradients}(x_{0:T}, a_{0:T-1}, R_{0:T-1})$.
\STATE Update $\theta \leftarrow \theta - \eta\Delta\theta$.
\STATE Update $\phi \leftarrow \phi - \eta\Delta\phi$.
\STATE Update $\rho \leftarrow \rho - \eta\Delta\rho$.

\STATE Input $x_{0:T}, a_{0:T-1}$ into $RB$ with priority $\sum_{j=0}^{T-d} R_j$.
\STATE If $|RB| > B$, remove episodes uniformly at random.
\STATE Sample $s_{0:T}$ from $RB$.
\STATE $\Delta\theta, \Delta\phi, \Delta\rho = \text{Gradients}(x_{0:T}, a_{0:T-1}, R_{0:T-1})$.
\STATE Update $\theta \leftarrow \theta - \eta\Delta\theta$.
\STATE Update $\phi \leftarrow \phi - \eta\Delta\phi$.
\STATE Update $\rho \leftarrow \rho - \eta\Delta\rho$.

\ENDFOR

\end{algorithmic}
\end{algorithm}

%% file: details.tex

For the algorithmic tasks, 
we follow a similar experimental setup as described in~\citet{pcl}.
We parameterize all values by a single 
LSTM recurrent neural network with internal dimension $128$
and multiple heads (one for each desired quantity).
At each training step, we sample a batch of $400$ episodes
using the current policy acting on the environment.
We perform a gradient step based on this batch.
We then add the experience to the replay buffer
and perform a gradient step based on an off-policy batch 
sampled from the replay buffer.
We fix the rollout to $d=10$.
As in~\citet{pcl}, our replay buffer is prioritized by 
episode rewards: the probability of sampling an episode from 
the replay buffer is $0.1 + 0.9 \cdot \exp\{\alpha R\} / Z$
where $R$ is the total reward of the episode, $Z$ is 
a normalizing factor, and we use $\alpha=0.5$.
We use a replay buffer of capacity $B=10,000$ episodes.
In our experiments we use a learning rate of $\eta=0.005$
and discount $\gamma=0.9$.

For HalfCheetah we parameterized the policy and value
networks as feed forward networks with two hidden layers
of dimension 64 and $\tanh$ non-linearities.
At each training step we sampled $100$ steps 
from the environment and input these into a
replay buffer.
We then sample a batch of 25 sub-episodes of 100 steps
from the replay buffer, prioritized by exponentiated recency
(with weight $0.01$) and perform a single training step.
We use rollout $d=10$, discount $\gamma=0.99$, and performed
a hyperparameter search over learning rate $\eta\in\{0.0005, 0.0001\}$.

In standard Soft PCL, the policy $\bar\mu_\theta$ is
determined by logits output by the neural network.
That is, 
\begin{equation}
\bar\mu_\theta(-|x) = \text{softmax}\left(\text{NN}(x, \theta)_{0:|\A|-1}\right),
\end{equation}
where $\text{NN}(x, \theta)_{0:|\A| - 1}$ are $|\A|$ output values of the neural network.
For Sparse PCL, to induce sparsity, we parameterize the policy
using the $\mathcal{G}$ function:
\begin{equation}
\bar\mu_\theta(-|x) = \text{relu}\left(
\text{NN}(x, \theta)_{0:|\A|-1} - \mathcal{G}(\text{NN}(x, \theta)_{0:|\A|-1})
\right).
\end{equation}
Accordingly, $\lambda_\theta$ is parameterized as
\begin{equation}
\lambda_\theta(-|x) = \text{relu}\left(
\mathcal{G}(\text{NN}(x, \theta)_{0:|\A|-1} - \text{NN}(x, \theta)_{0:|\A|-1}) \right)
\exp\{\text{NN}(x, \theta)_{|\A|} \}.
\end{equation}

\subsection{Experimental Results for ReversedAddition}

\begin{figure*}[h]
\begin{center}
  \setlength\tabcolsep{3pt}
  \renewcommand{\arraystretch}{0.2}
  \begin{tabular}{cccc}

    \multirow{2}{*}{\rotatebox[origin=c]{90}{\tiny ReversedAddition\hspace{0.5cm}}}
    & \tiny $|\A|=40$  &  \tiny $|\A|=64$  &  \tiny $|\A|=96$ \\
    &
    \includegraphics[width=0.16\columnwidth]{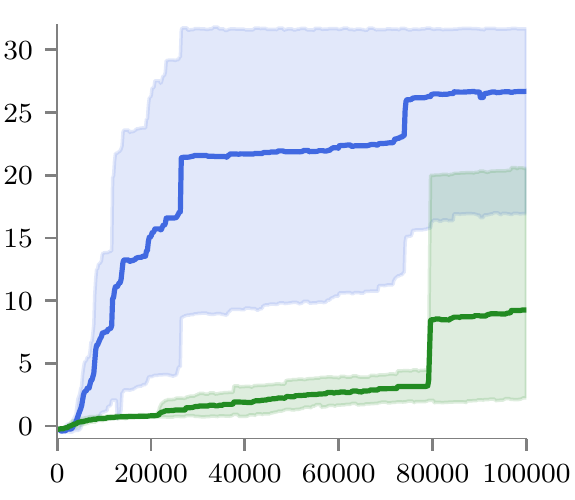} &
    \includegraphics[width=0.16\columnwidth]{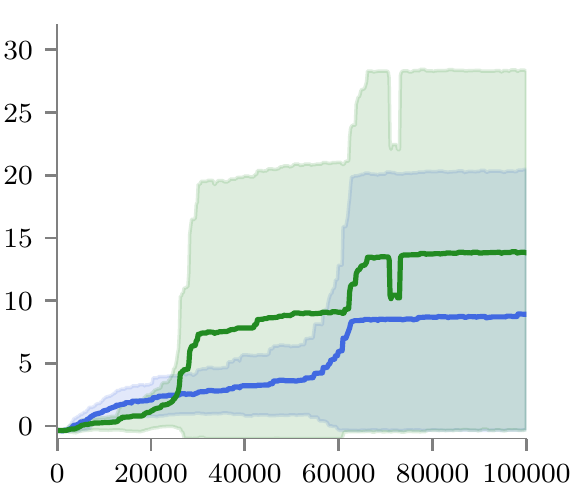} &
    \includegraphics[width=0.16\columnwidth]{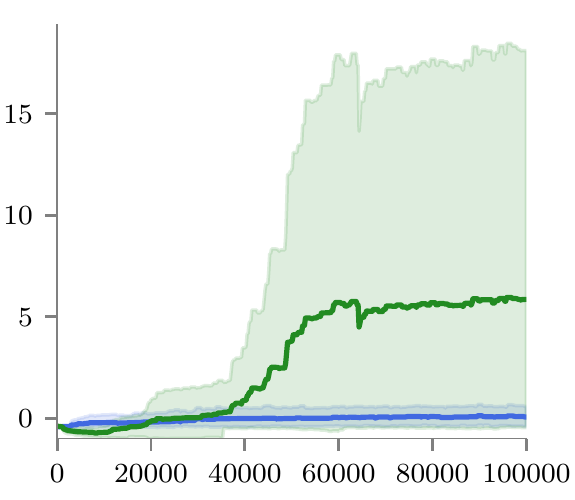} \\
     \multicolumn{4}{c}{\includegraphics[width=0.25\columnwidth]{rewards_-1_-1}}
  \end{tabular}
\end{center}
\caption{
The average reward over training for sparse PCL compared to the 
standard soft PCL on ReversedAddition.
In this task, the environment is a $2\times n$ grid of digits
representing two numbers in base-$|\mathcal{V}|$ that the agent needs to sum.
As in the other tasks in the main paper,
we see that sparse PCL becomes more advantageous
compared to soft PCL as the action space increases in size.
}
\label{fig:results2}
\end{figure*}